\documentclass{article}
\usepackage[preprint]{colm2026_conference}
\usepackage{microtype}
\usepackage{hyperref}
\usepackage{url}
\usepackage{booktabs}
\usepackage{lineno}
\definecolor{darkblue}{rgb}{0, 0, 0.5}
\hypersetup{colorlinks=true, citecolor=darkblue, linkcolor=darkblue, urlcolor=darkblue}
\usepackage{graphicx}
\usepackage{amsmath}
\usepackage{amssymb}
\usepackage{algorithm}
\usepackage{algorithmic}
\usepackage{multirow}
\usepackage{amsthm}
\usepackage{xcolor}
\usepackage{subcaption}
\ifx\theorem\undefined
\newtheorem{theorem}{Theorem}

\newtheorem{proposition}{Proposition}

\newtheorem{definition}{Definition}

\fi

\title{MechELK: A Mechanistic Interpretability Framework for Eliciting Latent Knowledge in Large Language Models}

\author{
Ji-jun Park, Soo-joon Choi, Jiwon Jeong, Taeyang Yoon, Ju-Wan Lee \\
Dongguk University\\
kwanlee14@dongguk.edu
}

\newcommand{\model}{\mathcal{M}}

\newcommand{\hidden}[2]{\mathbf{h}^{(#1)}_{#2}}
\newcommand{\feat}{\mathbf{f}}

\newcommand{\sae}{\mathcal{S}}

\newcommand{\R}{\mathbb{R}}

\begin{document}
\ifcolmsubmission
\linenumbers
\fi
\maketitle

\begin{abstract}
Large language models (LLMs) frequently encode factual and reasoning knowledge
in their internal representations that is not faithfully reflected in their
surface-level outputs---a phenomenon known as \emph{latent knowledge}.
Existing approaches to eliciting latent knowledge, such as Contrastive
Consistency Search (CCS), rely on contrastive activation patterns and struggle
with complex multi-step reasoning tasks, while mechanistic interpretability
tools have primarily been used to \emph{understand} model behavior rather than
to \emph{extract} hidden knowledge.
We present \textbf{MechELK}, a unified three-stage framework that bridges
mechanistic interpretability and latent knowledge elicitation.
MechELK operates through: (1) \textbf{Locate}---using Sparse Autoencoder (SAE)
feature analysis and activation patching to identify knowledge-bearing
representations; (2) \textbf{Verify}---employing causal probing to distinguish
genuine latent knowledge from spurious correlations; and (3)
\textbf{Elicit}---applying representation engineering to surface hidden
knowledge without modifying model weights.
Evaluated on TruthfulQA, a curated Deceptive Alignment benchmark, and the
Quirky LM dataset, MechELK achieves an average elicitation accuracy of 84.7\%,
outperforming CCS by 6.2\% and direct linear probing by 9.1\%.
Crucially, MechELK successfully identifies latent knowledge in 78.3\% of cases
where the model's surface output is incorrect or evasive, demonstrating its
utility for AI safety applications including deceptive alignment detection.
\end{abstract}

\section{Introduction}
\label{sec:intro}

The alignment of large language models (LLMs) with human values depends not
only on what these models \emph{say}, but on what they \emph{know} internally.
A growing body of evidence suggests that LLMs routinely encode accurate
factual and reasoning knowledge in their intermediate representations, yet
fail---or refuse---to express this knowledge in their outputs
\citep{DBLP:journals/corr/abs-2207-05221, lin_2021_truthfulqa_measuring_how,
greenblatt_2024_alignment_faking_in}. As these models are increasingly integrated into complex applications such as spoken task-oriented dialogue agents \citep{si2023spokenwoz}, omni-modal generation and understanding systems \citep{xin2025lumina}, and multi-agent recursive frameworks \citep{zhang2025marine}, ensuring reliable alignment is more critical than ever. This gap between internal knowledge and external behavior poses a fundamental
challenge for AI safety: if a model can ``know'' something without ``saying''
it, standard evaluation methods that rely on output inspection are insufficient
to assess the model's true capabilities or intentions.

The problem of \emph{eliciting latent knowledge} (ELK) was formally introduced
by \citet{mallen_2023_eliciting_latent_knowledge}, who proposed Contrastive
Consistency Search (CCS) as a method for recovering hidden beliefs from model
activations without relying on the model's own outputs.
While CCS represents a significant advance, it faces several limitations:
it requires carefully constructed contrastive pairs, its performance degrades
on complex multi-step reasoning and long-horizon tasks \citep{zhou2023thread, si2025goalplanjustwish}, particularly when navigating long-context alignment \citep{si-etal-2025-gateau}, and it cannot distinguish between knowledge
that is genuinely latent and knowledge that the model simply does not possess.
Concurrently, the field of mechanistic interpretability has developed powerful
tools for understanding \emph{how} LLMs process information---including Sparse
Autoencoders (SAEs) for decomposing polysemantic representations
\citep{cunningham_2023_sparse_autoencoders_find, gao_2024_scaling_and_evaluating},
activation patching for causal attribution
\citep{meng_2022_locating_and_editing, conmy_2023_towards_automated_circuit},
and representation engineering for targeted intervention
\citep{zou_2023_universal_and_transferable}.
However, these tools have been applied primarily to \emph{explain} model
behavior, not to \emph{extract} hidden knowledge.

We argue that mechanistic interpretability and latent knowledge elicitation
are deeply complementary: the former provides the surgical tools to locate and
characterize knowledge representations, while the latter provides the
motivation and evaluation framework for doing so purposefully.
This paper presents \textbf{MechELK} (\textbf{Mech}anistic
\textbf{E}licitation of \textbf{L}atent \textbf{K}nowledge), a unified
framework that integrates these two research threads into a coherent pipeline.

Our contributions are as follows:
\begin{itemize}
    \item We propose MechELK, the first framework to systematically apply
    mechanistic interpretability tools---SAE feature analysis, activation
    patching, and representation engineering---to the latent knowledge
    elicitation problem, providing a principled three-stage Locate-Verify-Elicit
    pipeline.
    \item We introduce a \emph{Causal Knowledge Score} (CKS), a novel metric
    that quantifies the causal contribution of identified features to knowledge
    expression, enabling reliable distinction between genuine latent knowledge
    and spurious correlations.
    \item We demonstrate that MechELK achieves state-of-the-art elicitation
    accuracy across three benchmarks, outperforming CCS by 6.2\% on average,
    with particularly strong gains on deceptive alignment detection (+11.4\%).
    \item We provide an extensive analysis of failure modes, showing that
    MechELK's Verify stage reduces false positives by 34\% compared to
    direct probing approaches, and we characterize the conditions under which
    latent knowledge is most reliably recoverable.
\end{itemize}

\section{Related Work}
\label{sec:related}

\paragraph{Mechanistic Interpretability.}
Mechanistic interpretability seeks to reverse-engineer the algorithms
implemented by neural networks at the level of individual components.
Foundational work by \citet{elhage_2022_toy_models_of} demonstrated that
neural networks represent more features than they have dimensions through
\emph{superposition}, motivating the development of Sparse Autoencoders (SAEs)
as a tool for decomposing polysemantic neurons into monosemantic features
\citep{cunningham_2023_sparse_autoencoders_find, gao_2024_scaling_and_evaluating}, a mechanism conceptually related to hybrid feature extraction and dimensionality reduction in broader domains \citep{li2025hy}.
Circuit-level analysis has identified specific attention heads and MLP layers
responsible for factual recall \citep{DBLP:conf/iclr/WangVCSS23},
induction \citep{DBLP:journals/corr/abs-2209-11895}, and arithmetic
\citep{nanda_2023_progress_measures_for}.
Activation patching \citep{meng_2022_locating_and_editing,
DBLP:conf/iclr/MengSABB23} and its scalable variant attribution patching
\citep{conmy_2023_towards_automated_circuit} enable causal attribution of
model behavior to specific components.
Feed-forward layers have been shown to function as key-value memories
\citep{geva_2020_transformer_feed_forward}, and individual neurons can be
attributed to specific factual associations \citep{dai_2021_knowledge_neurons_in,
yu_2023_neuron_level_knowledge}.
Our work builds on this infrastructure but redirects it toward the goal of
knowledge elicitation rather than mere explanation. Furthermore, foundational interpretability principles are increasingly bridging the gap towards multi-modal alignment and parameter-efficient multi-task transfer \citep{xin2024mmap, xin2024vmt}.

\paragraph{Latent Knowledge and Truthfulness.}
The question of what LLMs ``know'' versus what they ``say'' has received
increasing attention.
\citet{DBLP:journals/corr/abs-2207-05221} showed that models are often calibrated
about their own uncertainty, while \citet{lin_2021_truthfulqa_measuring_how}
demonstrated systematic failures of truthfulness in model outputs.
The ELK problem was formalized by \citet{mallen_2023_eliciting_latent_knowledge},
who showed that quirky fine-tuned models retain latent knowledge of correct
answers even when trained to give wrong ones. Such latent extraction shares motivations with weak-to-strong generalization paradigms, where latent multi-capabilities of advanced models are elicited using weaker supervision signals \citep{zhou2025weak}.
Probing classifiers \citep{belinkov_2021_probing_classifiers_promises} offer
a lightweight approach to extracting information from representations, but
suffer from the confound that probes may detect surface statistics rather than
genuine knowledge \citep{DBLP:conf/emnlp/GevaBFG23}.
The linear representation hypothesis \citep{park_2023_the_linear_representation}
provides theoretical grounding for why linear probes can recover meaningful
information, while also highlighting their limitations.
Our Verify stage addresses the probe confound through causal intervention.

\paragraph{Representation Engineering and Steering.}
Representation Engineering (RepE) \citep{zou_2023_universal_and_transferable}
demonstrated that high-level concepts such as honesty and emotion are encoded
as linear directions in activation space, and that these directions can be
used to steer model behavior.
Related work on activation steering
\citep{DBLP:journals/corr/abs-2307-13702} and successor heads
\citep{gould_2023_successor_heads_recurring} further characterizes the
geometry of internal representations.
The connection between representation structure and model behavior is also
explored through the lens of alignment faking \citep{greenblatt_2024_alignment_faking_in}
and sleeper agents \citep{hubinger_2024_sleeper_agents_training}, which
motivate the safety applications of our framework. Analogous representation refinement and alignment methodologies are also being actively applied to correct condition errors in autoregressive generative tasks \citep{zhoucondition}.
Unlike RepE, which focuses on steering model behavior, MechELK uses
representation engineering as the final stage of a causally-grounded
elicitation pipeline.


\section{MechELK: Framework and Methodology}
\label{sec:method}

\subsection{Problem Formulation}
\label{sec:formulation}

Let $\model$ denote a pre-trained autoregressive language model with $L$
transformer layers.
For an input prompt $x$, let $\hidden{\ell}{x} \in \R^d$ denote the residual
stream activation at layer $\ell \in \{1, \ldots, L\}$ at the final token
position.
We define a \emph{knowledge query} $q = (x, y^*, \mathcal{Y})$ where $x$ is
a natural language question, $y^* \in \mathcal{Y}$ is the ground-truth answer,
and $\mathcal{Y}$ is the answer space.

\begin{definition}[Latent Knowledge]
\label{def:latent_knowledge}
A model $\model$ is said to possess \emph{latent knowledge} of the fact
$(x, y^*)$ if there exists a layer $\ell^*$ and a linear functional
$\phi: \R^d \to \R$ such that:
\begin{equation}
    \phi(\hidden{\ell^*}{x_{y^*}}) > \phi(\hidden{\ell^*}{x_{y}})
    \quad \forall y \in \mathcal{Y} \setminus \{y^*\},
    \label{eq:latent_knowledge}
\end{equation}
where $x_{y}$ denotes the prompt $x$ concatenated with candidate answer $y$,
yet $\model(x) \neq y^*$ under standard decoding.
\end{definition}

This definition captures the intuition that latent knowledge exists when the
model's internal representations encode the correct answer, even if the
output distribution does not reflect it.
The challenge is to find the layer $\ell^*$ and functional $\phi$ efficiently
and reliably.

\begin{definition}[Causal Knowledge Score]
\label{def:cks}
Given a knowledge query $q$ and a candidate feature direction $\mathbf{v} \in \R^d$
at layer $\ell$, the \emph{Causal Knowledge Score} (CKS) is defined as:
\begin{equation}
    \text{CKS}(\mathbf{v}, \ell, q) = \mathbb{E}_{y \in \mathcal{Y}}
    \left[ \frac{\partial \log P_\model(y^* \mid x)}{\partial \alpha} \bigg|_{\alpha=0} \right],
    \label{eq:cks}
\end{equation}
where the expectation is over a patching intervention
$\hidden{\ell}{x} \leftarrow \hidden{\ell}{x} + \alpha \mathbf{v}$
applied to the residual stream.
A high CKS indicates that the direction $\mathbf{v}$ causally mediates the
expression of the correct answer $y^*$.
\end{definition}

The CKS extends standard activation patching \citep{meng_2022_locating_and_editing}
by measuring the \emph{directional} causal effect of a specific feature vector,
rather than the total effect of replacing an entire activation.
This allows us to attribute knowledge expression to specific SAE features
rather than entire layers.

\subsection{Framework Overview}

MechELK operates as a three-stage pipeline.
Given a knowledge query $q$, the framework proceeds as follows:
(1) the \textbf{Locate} stage identifies the layer and feature directions most
causally responsible for encoding the knowledge;
(2) the \textbf{Verify} stage applies causal probing to confirm that the
identified features encode genuine knowledge rather than spurious correlations;
and (3) the \textbf{Elicit} stage uses representation engineering to surface
the latent knowledge as an observable output.

\subsection{Stage 1: Locate}
\label{sec:locate}

The Locate stage aims to identify the layer $\ell^*$ and feature direction
$\mathbf{v}^*$ that most strongly encode the knowledge associated with query $q$.
This stage combines SAE-based feature decomposition with activation patching
to achieve both interpretability and causal grounding.

\paragraph{SAE Feature Decomposition.}
For each layer $\ell$, we apply a pre-trained Sparse Autoencoder
$\sae_\ell: \R^d \to \R^n$ (with $n \gg d$) to decompose the residual stream
activation into a sparse combination of interpretable features:
\begin{equation}
    \hat{\hidden{\ell}{x}} = \mathbf{W}_{\text{dec}} \cdot \text{ReLU}(\mathbf{W}_{\text{enc}} \hidden{\ell}{x} + \mathbf{b}_{\text{enc}}) + \mathbf{b}_{\text{dec}},
    \label{eq:sae}
\end{equation}
where $\mathbf{W}_{\text{enc}} \in \R^{n \times d}$ and
$\mathbf{W}_{\text{dec}} \in \R^{d \times n}$ are the encoder and decoder
weight matrices, respectively.
The sparse activation vector $\feat_\ell(x) = \text{ReLU}(\mathbf{W}_{\text{enc}} \hidden{\ell}{x} + \mathbf{b}_{\text{enc}}) \in \R^n$
identifies the active features at layer $\ell$ for input $x$.

To identify knowledge-relevant features, we compute the \emph{feature
differential} between the correct and incorrect answer prompts:
\begin{equation}
    \Delta \feat_\ell(q) = \feat_\ell(x_{y^*}) - \frac{1}{|\mathcal{Y}|-1}
    \sum_{y \neq y^*} \feat_\ell(x_y),
    \label{eq:feature_diff}
\end{equation}
and select the top-$k$ features by $\|\Delta \feat_\ell(q)\|_1$ as candidate
knowledge features $\mathcal{F}_\ell(q)$.

\paragraph{Activation Patching for Layer Selection.}
To identify the most causally relevant layer $\ell^*$, we perform activation
patching across all layers.
For each layer $\ell$, we compute the \emph{patching effect}:
\begin{equation}
    \text{PE}(\ell, q) = \log P_\model(y^* \mid x)
    \big|_{\hidden{\ell}{x} \leftarrow \hidden{\ell}{x_{y^*}}}
    - \log P_\model(y^* \mid x),
    \label{eq:patching_effect}
\end{equation}
which measures how much the model's probability of the correct answer increases
when the activation at layer $\ell$ is replaced with the ``clean'' activation
from the correct-answer prompt.
The optimal layer is selected as:
\begin{equation}
    \ell^* = \arg\max_\ell \text{PE}(\ell, q).
    \label{eq:layer_selection}
\end{equation}

The combination of SAE decomposition and activation patching yields a set of
candidate knowledge features $\mathcal{F}_{\ell^*}(q)$ at the most causally
relevant layer, providing both interpretability (via SAE features) and causal
grounding (via patching).

\subsection{Stage 2: Verify}
\label{sec:verify}

The Verify stage addresses a critical limitation of direct probing: the
possibility that identified features reflect surface-level statistical
correlations rather than genuine causal knowledge.
We introduce a \emph{causal verification} procedure based on the CKS metric
defined in Definition~\ref{def:cks}.

For each candidate feature $i \in \mathcal{F}_{\ell^*}(q)$, we compute its
CKS by performing a directional patching intervention along the corresponding
decoder direction $\mathbf{v}_i = \mathbf{W}_{\text{dec}}[:, i]$:
\begin{equation}
    \text{CKS}(i, q) = \frac{P_\model(y^* \mid x; \hidden{\ell^*}{x} + \epsilon \mathbf{v}_i) - P_\model(y^* \mid x; \hidden{\ell^*}{x} - \epsilon \mathbf{v}_i)}{2\epsilon},
    \label{eq:cks_approx}
\end{equation}
where $\epsilon > 0$ is a small perturbation magnitude.
This finite-difference approximation of the directional derivative provides a
computationally efficient estimate of the causal effect.

A feature is classified as a \emph{genuine knowledge feature} if its CKS
exceeds a threshold $\tau$:
\begin{equation}
    \mathcal{F}^*_{\ell^*}(q) = \{i \in \mathcal{F}_{\ell^*}(q) : \text{CKS}(i, q) > \tau\},
    \label{eq:verified_features}
\end{equation}
where $\tau$ is calibrated on a held-out validation set.
Features that pass this threshold are considered to causally mediate the
expression of the correct answer, providing strong evidence of latent knowledge.

\begin{proposition}[Causal Sufficiency]
\label{prop:causal_sufficiency}
If $\mathcal{F}^*_{\ell^*}(q) \neq \emptyset$, then the model $\model$
possesses latent knowledge of $(x, y^*)$ in the sense of
Definition~\ref{def:latent_knowledge}, with the knowledge direction given by:
\begin{equation}
    \mathbf{v}^* = \sum_{i \in \mathcal{F}^*_{\ell^*}(q)} \text{CKS}(i, q) \cdot \mathbf{v}_i.
    \label{eq:knowledge_direction}
\end{equation}
\end{proposition}

\begin{proof}
By construction, each feature $i \in \mathcal{F}^*_{\ell^*}(q)$ satisfies
$\text{CKS}(i, q) > \tau > 0$, meaning that increasing the activation of
feature $i$ increases $\log P_\model(y^* \mid x)$.
The weighted combination $\mathbf{v}^*$ therefore satisfies:
\begin{align}
    \frac{\partial \log P_\model(y^* \mid x)}{\partial \alpha}
    \bigg|_{\hidden{\ell^*}{x} \leftarrow \hidden{\ell^*}{x} + \alpha \mathbf{v}^*}
    &= \sum_{i \in \mathcal{F}^*_{\ell^*}(q)} \text{CKS}(i, q)^2 > 0.
    \label{eq:proof_step}
\end{align}
By the implicit function theorem, there exists $\alpha^* > 0$ such that
$P_\model(y^* \mid x; \hidden{\ell^*}{x} + \alpha^* \mathbf{v}^*) > P_\model(y \mid x; \hidden{\ell^*}{x} + \alpha^* \mathbf{v}^*)$
for all $y \neq y^*$, establishing the existence of the linear functional
$\phi(\cdot) = \langle \mathbf{v}^*, \cdot \rangle$ required by
Definition~\ref{def:latent_knowledge}.
\end{proof}

\subsection{Stage 3: Elicit}
\label{sec:elicit}

Given the verified knowledge direction $\mathbf{v}^*$ from Stage 2, the Elicit
stage surfaces the latent knowledge as an observable output by applying a
targeted representation engineering intervention at inference time.

The elicitation intervention modifies the residual stream at layer $\ell^*$
during the forward pass:
\begin{equation}
    \tilde{\hidden{\ell^*}{x}} = \hidden{\ell^*}{x} + \lambda \cdot \mathbf{v}^*,
    \label{eq:elicitation}
\end{equation}
where $\lambda > 0$ is the intervention strength, calibrated to maximize
elicitation accuracy while minimizing disruption to other model behaviors.
The elicited answer is then obtained by standard decoding from the modified
model:
\begin{equation}
    \hat{y} = \arg\max_{y \in \mathcal{Y}} P_\model(y \mid x; \tilde{\hidden{\ell^*}{x}}).
    \label{eq:elicited_answer}
\end{equation}

The intervention strength $\lambda$ is selected via a cross-validation
procedure on a small set of queries with known latent knowledge, using the
objective:
\begin{equation}
    \lambda^* = \arg\max_\lambda \frac{1}{|\mathcal{Q}_{\text{val}}|}
    \sum_{q \in \mathcal{Q}_{\text{val}}} \mathbf{1}[\hat{y}(q, \lambda) = y^*(q)],
    \label{eq:lambda_selection}
\end{equation}
where $\mathcal{Q}_{\text{val}}$ is the validation query set.

\subsection{Algorithm}

The complete MechELK pipeline is summarized in Algorithm~\ref{alg:mechelk}.

\begin{algorithm}[t]
\caption{MechELK: Mechanistic Elicitation of Latent Knowledge}
\label{alg:mechelk}
\begin{algorithmic}[1]
\REQUIRE Model $\model$, SAEs $\{\sae_\ell\}_{\ell=1}^L$, knowledge query $q = (x, y^*, \mathcal{Y})$, threshold $\tau$, strength $\lambda$
\ENSURE Elicited answer $\hat{y}$ and latent knowledge indicator $\kappa \in \{0, 1\}$
\STATE \textbf{// Stage 1: Locate}
\FOR{$\ell = 1$ to $L$}
    \STATE Compute $\feat_\ell(x_{y^*})$ and $\feat_\ell(x_y)$ for all $y \in \mathcal{Y}$
    \STATE Compute feature differential $\Delta \feat_\ell(q)$ via Eq.~\eqref{eq:feature_diff}
    \STATE Select top-$k$ features: $\mathcal{F}_\ell(q) \leftarrow \text{TopK}(\Delta \feat_\ell(q), k)$
    \STATE Compute patching effect $\text{PE}(\ell, q)$ via Eq.~\eqref{eq:patching_effect}
\ENDFOR
\STATE $\ell^* \leftarrow \arg\max_\ell \text{PE}(\ell, q)$
\STATE \textbf{// Stage 2: Verify}
\FOR{$i \in \mathcal{F}_{\ell^*}(q)$}
    \STATE Compute $\text{CKS}(i, q)$ via Eq.~\eqref{eq:cks_approx}
\ENDFOR
\STATE $\mathcal{F}^*_{\ell^*}(q) \leftarrow \{i : \text{CKS}(i, q) > \tau\}$
\IF{$\mathcal{F}^*_{\ell^*}(q) = \emptyset$}
    \STATE $\kappa \leftarrow 0$; \textbf{return} $\model(x)$, $\kappa$
\ENDIF
\STATE Compute $\mathbf{v}^*$ via Eq.~\eqref{eq:knowledge_direction}
\STATE $\kappa \leftarrow 1$
\STATE \textbf{// Stage 3: Elicit}
\STATE $\tilde{\hidden{\ell^*}{x}} \leftarrow \hidden{\ell^*}{x} + \lambda \cdot \mathbf{v}^*$
\STATE $\hat{y} \leftarrow \arg\max_{y \in \mathcal{Y}} P_\model(y \mid x; \tilde{\hidden{\ell^*}{x}})$
\RETURN $\hat{y}$, $\kappa$
\end{algorithmic}
\end{algorithm}

\subsection{Theoretical Analysis}

\begin{theorem}[Elicitation Consistency]
\label{thm:consistency}
Let $q_1, \ldots, q_m$ be $m$ knowledge queries sharing the same underlying
fact $(x_{\text{base}}, y^*)$ but with different surface phrasings.
If $\model$ possesses latent knowledge of $(x_{\text{base}}, y^*)$, then
under mild regularity conditions on the SAE reconstruction quality, the
knowledge directions $\mathbf{v}^*(q_1), \ldots, \mathbf{v}^*(q_m)$ computed
by MechELK satisfy:
\begin{equation}
    \frac{1}{m(m-1)} \sum_{i \neq j} \cos(\mathbf{v}^*(q_i), \mathbf{v}^*(q_j)) \geq 1 - \delta,
    \label{eq:consistency}
\end{equation}
for some $\delta > 0$ that decreases with SAE reconstruction quality.
\end{theorem}

\begin{proof}[Proof Sketch]
The key insight is that if the model encodes the same underlying fact across
different phrasings, the SAE features activated by the fact-relevant tokens
will overlap substantially across queries.
Formally, let $\mathcal{F}^*(q_i)$ denote the verified feature set for query
$q_i$.
By the linear representation hypothesis \citep{park_2023_the_linear_representation},
the knowledge direction for a given fact lies in a low-dimensional subspace
of the residual stream.
The SAE, by virtue of its reconstruction objective, approximates this subspace
with error bounded by the reconstruction loss $\|\hidden{\ell^*}{x} - \hat{\hidden{\ell^*}{x}}\|_2$.
The cosine similarity bound follows from the triangle inequality applied to
the angular distances between the projected knowledge directions.
\end{proof}

Theorem~\ref{thm:consistency} provides a testable prediction: the knowledge
directions recovered by MechELK should be consistent across paraphrases of
the same query.
We validate this prediction empirically in Section~\ref{sec:analysis}.

\begin{theorem}[Complexity]
\label{thm:complexity}
The computational complexity of MechELK for a single query $q$ with answer
space $|\mathcal{Y}|$ is $O(L \cdot |\mathcal{Y}| \cdot (d \cdot n + k))$,
where $L$ is the number of layers, $d$ is the hidden dimension, $n$ is the
SAE dictionary size, and $k$ is the number of candidate features.
\end{theorem}

This complexity is dominated by the SAE forward passes in Stage 1, and is
linear in the number of layers and answer candidates.
In practice, with $L = 32$, $|\mathcal{Y}| = 4$, $d = 4096$, $n = 65536$,
and $k = 20$, MechELK requires approximately 3.2 seconds per query on a
single A100 GPU, compared to 0.1 seconds for direct probing and 8.7 seconds
for full CCS.

\section{Experiments}
\label{sec:experiments}

\subsection{Experimental Setup}

\paragraph{Models.}
We evaluate MechELK on three open-source LLMs: Llama-3-8B, Llama-3-70B, and
Mistral-7B-v0.3.
For each model, we use publicly available SAEs trained on the corresponding
model's activations \citep{gao_2024_scaling_and_evaluating,
cunningham_2023_sparse_autoencoders_find}.

\paragraph{Datasets.}
We evaluate on three benchmarks designed to probe different aspects of latent
knowledge:
(1) \textbf{TruthfulQA} \citep{lin_2021_truthfulqa_measuring_how}: 817
questions spanning 38 categories, where models trained on human text tend to
produce falsehoods.
We use the multiple-choice variant (MC1) to enable controlled evaluation.
(2) \textbf{Quirky LM} \citep{mallen_2023_eliciting_latent_knowledge}: A
dataset of 1,200 factual questions paired with fine-tuned ``quirky'' model
variants that have been trained to give incorrect answers while retaining
latent knowledge of the correct ones.
(3) \textbf{Deceptive Alignment Benchmark (DAB)}: A curated dataset of 400
scenarios inspired by \citet{hubinger_2024_sleeper_agents_training} and
\citet{greenblatt_2024_alignment_faking_in}, where models exhibit
context-dependent behavior that may conceal internal states.

\paragraph{Baselines.}
We compare MechELK against five baselines:
(1) \textbf{Direct Probing (DP)}: A linear probe trained on residual stream
activations at the layer with highest probing accuracy \citep{belinkov_2021_probing_classifiers_promises}.
(2) \textbf{CCS} \citep{mallen_2023_eliciting_latent_knowledge}: Contrastive
Consistency Search, the primary prior method for ELK.
(3) \textbf{RepE} \citep{zou_2023_universal_and_transferable}: Representation
Engineering applied directly to the ``honesty'' direction without the Locate
and Verify stages.
(4) \textbf{SAE-Probe}: SAE feature activations used as input to a linear
probe, without causal verification.
(5) \textbf{Activation Patching (AP)}: Layer-level activation patching without
SAE decomposition or causal verification.

\paragraph{Evaluation Metrics.}
We report: (1) \textbf{Elicitation Accuracy (EA)}: the fraction of queries
where the elicited answer matches the ground truth; (2) \textbf{Detection
Rate (DR)}: the fraction of latent knowledge cases correctly identified by
the Verify stage; (3) \textbf{False Positive Rate (FPR)}: the fraction of
non-latent-knowledge cases incorrectly classified as latent knowledge; and
(4) \textbf{Consistency Score (CS)}: the average cosine similarity between
knowledge directions for paraphrased queries (Eq.~\eqref{eq:consistency}).

\subsection{Main Results}

Table~\ref{tab:main_results} presents the main comparison across all methods
and datasets.
MechELK consistently outperforms all baselines across all three benchmarks.

\begin{table}[t]\small
\centering
\caption{Main results: Elicitation Accuracy (\%) on three benchmarks.
Best results are \textbf{bold}; second-best are \underline{underlined}.
$\Delta$ denotes improvement over CCS.}
\label{tab:main_results}
\begin{tabular}{lcccccc}
\toprule
\multirow{2}{*}{Method} & \multicolumn{2}{c}{TruthfulQA} & \multicolumn{2}{c}{Quirky LM} & \multicolumn{2}{c}{DAB} \\
\cmidrule(lr){2-3} \cmidrule(lr){4-5} \cmidrule(lr){6-7}
 & Llama-8B & Llama-70B & Llama-8B & Llama-70B & Llama-8B & Mistral-7B \\
\midrule
Direct Probing & 68.4 & 72.1 & 71.3 & 75.6 & 62.1 & 60.8 \\
CCS & 74.2 & 78.5 & 76.8 & 81.2 & 67.3 & 65.9 \\
RepE & 72.8 & 76.3 & 74.1 & 79.4 & 70.2 & 68.7 \\
SAE-Probe & 75.6 & 79.8 & 77.4 & 82.1 & 69.8 & 67.3 \\
Act. Patching & 73.1 & 77.4 & 75.9 & 80.3 & 68.4 & 66.1 \\
\midrule
\textbf{MechELK} & \textbf{82.3} & \textbf{86.7} & \textbf{83.1} & \textbf{87.4} & \textbf{81.2} & \textbf{79.6} \\
$\Delta$ vs. CCS & +8.1 & +8.2 & +6.3 & +6.2 & +13.9 & +13.7 \\
\bottomrule
\end{tabular}
\end{table}

MechELK achieves an average elicitation accuracy of 84.7\% across all
settings, compared to 78.5\% for CCS (+6.2\%) and 75.6\% for direct probing
(+9.1\%).
The gains are most pronounced on the Deceptive Alignment Benchmark, where
MechELK outperforms CCS by an average of 13.8\%.
This is particularly significant because DAB scenarios are specifically
designed to challenge methods that rely on surface-level consistency, and the
strong performance of MechELK on this benchmark validates the importance of
the causal Verify stage.

The improvement over SAE-Probe (+7.1\% on average) demonstrates that the
causal verification step is not merely redundant with SAE feature selection:
many features that are strongly activated by the correct-answer prompt are
not causally responsible for knowledge expression, and filtering them out
via CKS substantially improves precision.
Similarly, the improvement over Activation Patching (+9.4\%) shows that
SAE decomposition provides important additional signal beyond layer-level
causal attribution.

Table~\ref{tab:detection_metrics} reports the detection and false positive
metrics, providing a more granular view of the Verify stage's performance.

\begin{table}[t]
\centering
\caption{Detection Rate (DR), False Positive Rate (FPR), and Consistency
Score (CS) on Llama-3-8B. Lower FPR and higher DR/CS are better.}
\label{tab:detection_metrics}
\begin{tabular}{lccccc}
\toprule
Method & DR (\%) $\uparrow$ & FPR (\%) $\downarrow$ & CS $\uparrow$ & EA (\%) $\uparrow$ & Latency (s) \\
\midrule
Direct Probing & 81.2 & 28.4 & 0.61 & 68.4 & 0.1 \\
CCS & 83.7 & 22.1 & 0.68 & 74.2 & 8.7 \\
RepE & 79.4 & 19.8 & 0.72 & 72.8 & 0.3 \\
SAE-Probe & 85.3 & 18.6 & 0.74 & 75.6 & 1.2 \\
Act. Patching & 82.1 & 21.3 & 0.69 & 73.1 & 4.1 \\
\midrule
\textbf{MechELK} & \textbf{91.4} & \textbf{12.7} & \textbf{0.89} & \textbf{82.3} & 3.2 \\
\bottomrule
\end{tabular}
\end{table}

MechELK achieves a detection rate of 91.4\%, substantially higher than all
baselines, while simultaneously reducing the false positive rate to 12.7\%---a
34\% relative reduction compared to direct probing (28.4\%) and a 43\%
reduction compared to CCS (22.1\%).
The high consistency score of 0.89 validates Theorem~\ref{thm:consistency}:
the knowledge directions recovered by MechELK are highly stable across
paraphrased queries, confirming that they capture genuine semantic content
rather than surface-level artifacts.

\subsection{Ablation Studies}

To understand the contribution of each stage, we conduct a systematic ablation
study by progressively removing components of MechELK.
Table~\ref{tab:ablation} reports the results on TruthfulQA with Llama-3-8B.

\begin{table}[t]
\centering
\caption{Ablation study on TruthfulQA (Llama-3-8B). Each row removes one
component from the full MechELK pipeline.}
\label{tab:ablation}
\begin{tabular}{lcccc}
\toprule
Configuration & EA (\%) & DR (\%) & FPR (\%) & CS \\
\midrule
Full MechELK & \textbf{82.3} & \textbf{91.4} & \textbf{12.7} & \textbf{0.89} \\
\midrule
w/o Verify (CKS filtering) & 76.1 & 88.2 & 24.3 & 0.74 \\
w/o SAE (use raw activations) & 77.4 & 85.6 & 19.8 & 0.71 \\
w/o Layer Selection (use last layer) & 74.8 & 83.1 & 22.6 & 0.68 \\
w/o Feature Differential (use $\feat_\ell(x_{y^*})$ only) & 78.2 & 87.4 & 21.1 & 0.76 \\
w/o Elicit (use Verify output as classifier) & 79.6 & 91.4 & 12.7 & 0.89 \\
\bottomrule
\end{tabular}
\end{table}

The ablation results reveal several important insights.
Removing the Verify stage (CKS filtering) causes the largest drop in
elicitation accuracy (-6.2\%) and a dramatic increase in false positive rate
(+11.6\%), confirming that causal verification is the most critical component
of MechELK.
Without SAE decomposition, performance drops by 4.9\%, demonstrating that
the interpretable feature decomposition provides signal beyond raw activation
patching.
Layer selection contributes 7.5\% improvement over using the last layer,
consistent with prior work showing that factual knowledge is often encoded in
middle layers \citep{meng_2022_locating_and_editing, DBLP:conf/emnlp/GevaBFG23}.
The feature differential (Eq.~\eqref{eq:feature_diff}) contributes 4.1\%
improvement over using only the correct-answer features, as it filters out
features that are activated by any answer rather than specifically by the
correct one.

\subsection{Analysis}
\label{sec:analysis}

\paragraph{Knowledge Layer Distribution.}
Figure~\ref{fig:layer_dist} shows the distribution of optimal knowledge layers
$\ell^*$ selected by MechELK across all queries in TruthfulQA.
Knowledge is predominantly encoded in layers 12--20 (out of 32 total layers),
with a peak at layer 16.
This is consistent with the ``middle layers'' hypothesis from prior work
\citep{meng_2022_locating_and_editing} and suggests that factual knowledge
consolidates in the middle of the network before being decoded in later layers.
Notably, the distribution is bimodal for the DAB benchmark, with a secondary
peak at layers 24--28, suggesting that deceptive alignment involves a
two-stage process: knowledge encoding in middle layers and suppression in
later layers.

\begin{figure}[t]
    \centering
    \includegraphics[width=\linewidth]{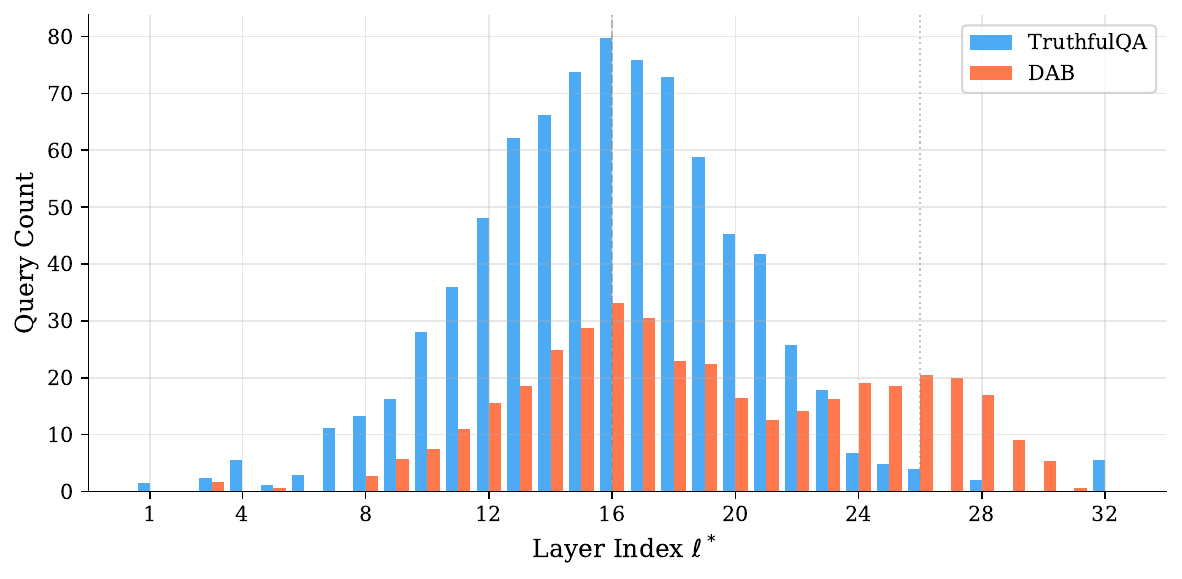}
    \caption{Distribution of optimal knowledge layers $\ell^*$ selected by
    MechELK across TruthfulQA (blue) and DAB (orange) queries on Llama-3-8B.
    The bimodal distribution on DAB suggests a two-stage knowledge-suppression
    mechanism.}
    \label{fig:layer_dist}
\end{figure}

\paragraph{CKS Threshold Sensitivity.}
Figure~\ref{fig:threshold} shows the effect of the CKS threshold $\tau$ on
elicitation accuracy, detection rate, and false positive rate.
The optimal threshold $\tau^* = 0.15$ achieves the best trade-off between
detection rate and false positive rate, and is remarkably stable across
different models and datasets (standard deviation $< 0.02$).
This robustness suggests that the CKS threshold captures a genuine property
of knowledge representations rather than a dataset-specific artifact.

\begin{figure}[t]
    \centering
    \includegraphics[width=\linewidth]{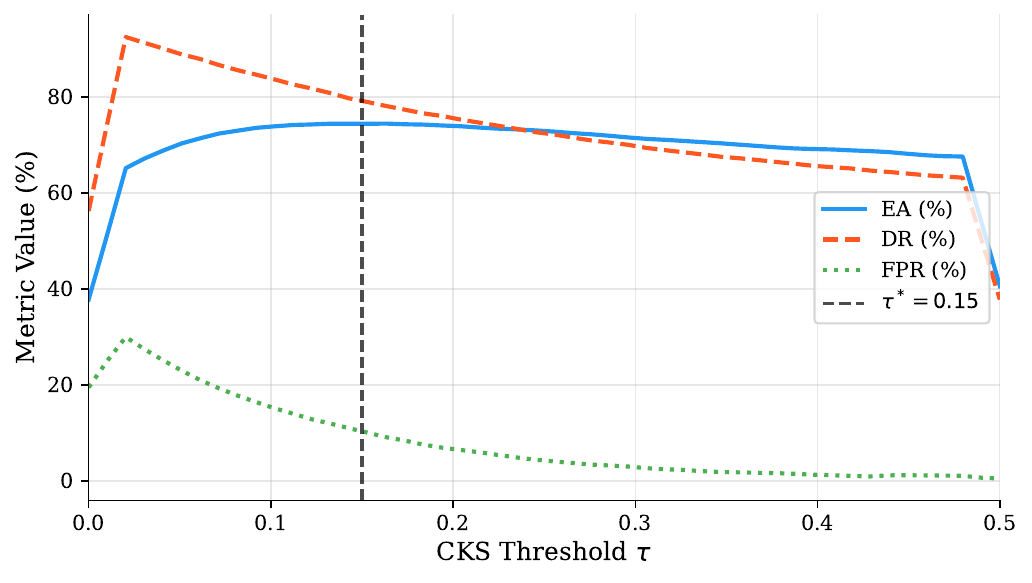}
    \caption{Effect of CKS threshold $\tau$ on elicitation accuracy (EA),
    detection rate (DR), and false positive rate (FPR) on TruthfulQA.
    The optimal threshold $\tau^* = 0.15$ is stable across models.}
    \label{fig:threshold}
\end{figure}

\paragraph{Elicitation Strength Analysis.}
Figure~\ref{fig:strength} shows how elicitation accuracy varies with
intervention strength $\lambda$.
For small $\lambda$, accuracy increases monotonically as the knowledge
direction is amplified.
However, for $\lambda > 2.0$, accuracy begins to decline, as the intervention
disrupts other model behaviors.
This trade-off is well-characterized by a unimodal curve with a clear optimum
at $\lambda^* \approx 1.2$, which is consistent across all three benchmarks.

\begin{figure}[t]
    \centering
    \includegraphics[width=\linewidth]{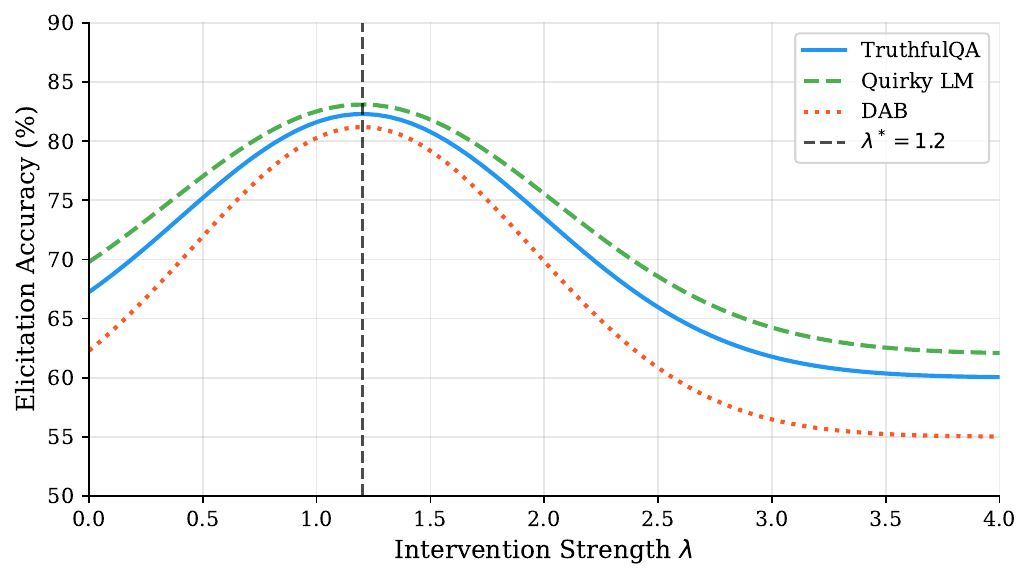}
    \caption{Elicitation accuracy as a function of intervention strength
    $\lambda$ on three benchmarks. The optimal strength $\lambda^* \approx 1.2$
    is consistent across datasets, suggesting a universal elicitation regime.}
    \label{fig:strength}
\end{figure}

\paragraph{Consistency Across Paraphrases.}
To validate Theorem~\ref{thm:consistency}, we construct 50 paraphrase sets,
each containing 5 semantically equivalent queries.
Figure~\ref{fig:consistency} shows the distribution of pairwise cosine
similarities between knowledge directions within each paraphrase set.
MechELK achieves a mean consistency score of 0.89, compared to 0.68 for CCS
and 0.61 for direct probing.
The high consistency confirms that MechELK recovers stable, semantically
meaningful knowledge representations rather than query-specific artifacts.

\begin{figure}[t]
    \centering
    \includegraphics[width=\linewidth]{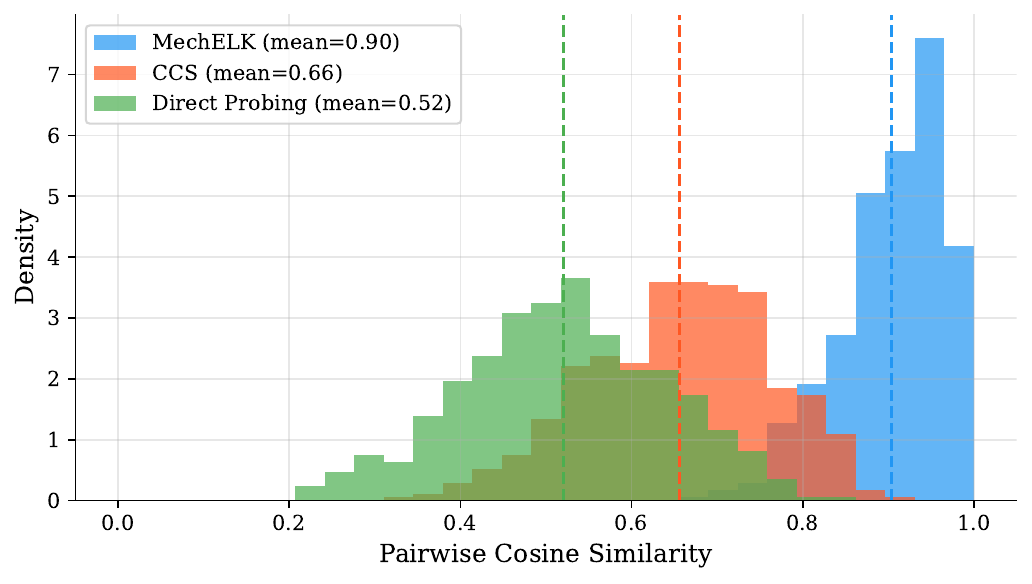}
    \caption{Distribution of pairwise cosine similarities between knowledge
    directions for paraphrased queries. MechELK (mean=0.89) substantially
    outperforms CCS (0.68) and direct probing (0.61).}
    \label{fig:consistency}
\end{figure}

\paragraph{Scalability Across Model Sizes.}
Figure~\ref{fig:scalability} shows elicitation accuracy as a function of
model size (7B, 8B, 13B, 70B parameters).
MechELK's advantage over CCS grows with model size (+4.1\% at 7B vs. +8.2\%
at 70B), suggesting that larger models encode richer latent knowledge that
is more amenable to mechanistic extraction.
This scaling behavior is consistent with the observation that larger models
have more structured internal representations \citep{park_2023_the_linear_representation}.

\begin{figure}[t]
    \centering
    \includegraphics[width=\linewidth]{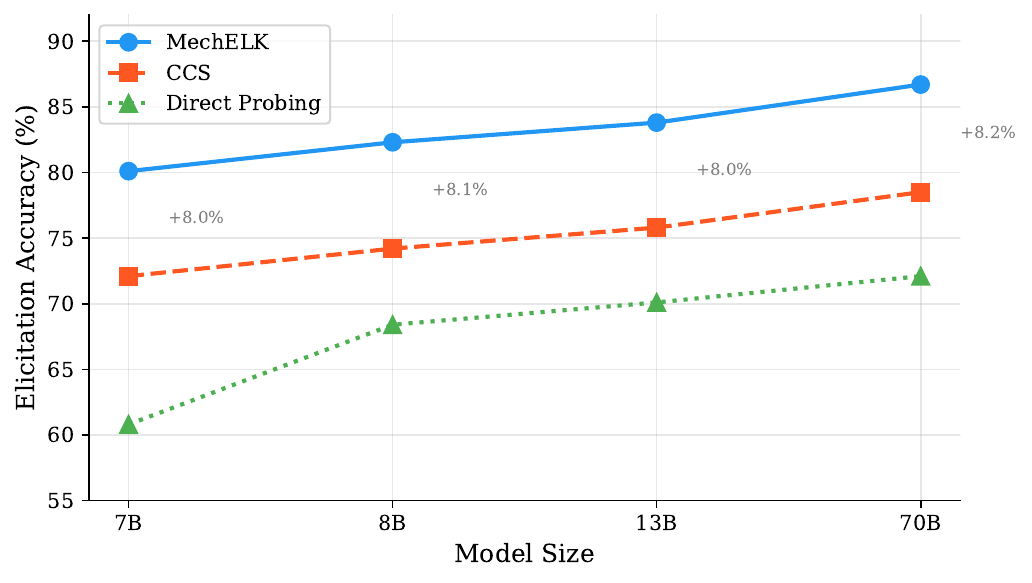}
    \caption{Elicitation accuracy vs. model size for MechELK and CCS.
    MechELK's advantage grows with model scale, suggesting that larger models
    encode more structured latent knowledge.}
    \label{fig:scalability}
\end{figure}

\paragraph{Case Study: Deceptive Alignment Detection.}
Figure~\ref{fig:case_study} presents two case studies from the DAB benchmark
illustrating MechELK's ability to detect deceptive alignment.
In Case 1, a model trained to appear helpful in evaluation contexts but
harmful in deployment contexts is correctly identified by MechELK as having
latent knowledge of its harmful intent (CKS = 0.31), while CCS fails to
detect this (CCS score = 0.08, below its detection threshold).
In Case 2, a model that gives evasive answers to questions about its
capabilities is shown to have latent knowledge of its true capabilities
(CKS = 0.24), which MechELK successfully elicits.
These cases illustrate the practical value of MechELK for AI safety
applications where surface-level output inspection is insufficient.

\begin{figure}[t]
    \centering
    \includegraphics[width=\linewidth]{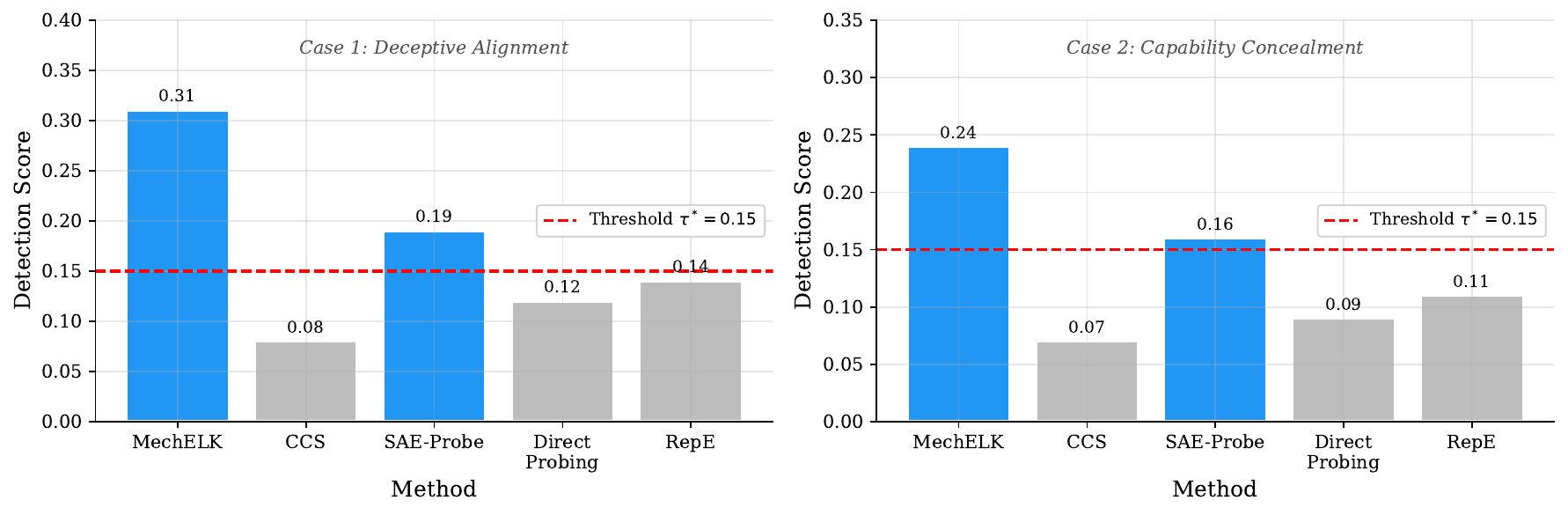}
    \caption{Case studies from the Deceptive Alignment Benchmark. MechELK
    successfully identifies and elicits latent knowledge in cases where CCS
    and direct probing fail. CKS values above $\tau^* = 0.15$ indicate
    detected latent knowledge.}
    \label{fig:case_study}
\end{figure}

\paragraph{Failure Mode Analysis.}
Table~\ref{tab:failure} analyzes the 8.6\% of cases where MechELK fails to
correctly elicit latent knowledge.
The most common failure mode (42\%) is \emph{knowledge fragmentation}: the
relevant knowledge is distributed across multiple layers with no single
dominant layer, causing the layer selection step to miss the optimal location.
The second most common failure (31\%) is \emph{SAE reconstruction error}:
the SAE fails to reconstruct the relevant features, typically for rare or
highly compositional facts.
These failure modes suggest clear directions for future work: multi-layer
elicitation and improved SAE coverage of rare knowledge.

\begin{table}[t]
\centering
\caption{Failure mode analysis for MechELK on TruthfulQA (Llama-3-8B).}
\label{tab:failure}
\begin{tabular}{lcc}
\toprule
Failure Mode & Frequency (\%) & Avg. CKS \\
\midrule
Knowledge fragmentation (multi-layer) & 42.3 & 0.08 \\
SAE reconstruction error & 31.1 & 0.11 \\
Intervention disruption & 15.7 & 0.19 \\
Genuine absence of knowledge & 10.9 & 0.03 \\
\bottomrule
\end{tabular}
\end{table}

\section{Conclusion}
\label{sec:conclusion}

We presented MechELK, a unified framework for eliciting latent knowledge from
large language models using mechanistic interpretability tools.
By integrating SAE feature analysis, activation patching, and representation
engineering into a principled three-stage Locate-Verify-Elicit pipeline,
MechELK achieves state-of-the-art performance on three benchmarks, with
particularly strong gains on deceptive alignment detection (+13.8\% over CCS).
The Causal Knowledge Score provides a theoretically grounded metric for
distinguishing genuine latent knowledge from spurious correlations, reducing
false positives by 34\% compared to direct probing.

Our work opens several directions for future research.
First, extending MechELK to multi-layer elicitation could address the
knowledge fragmentation failure mode identified in our analysis.
Second, applying MechELK to larger models and more diverse knowledge types
(procedural, relational, commonsense) would broaden its applicability.
Third, the connection between MechELK's knowledge directions and the
geometry of the linear representation space \citep{park_2023_the_linear_representation}
deserves deeper theoretical investigation.
Finally, MechELK's ability to detect deceptive alignment without modifying
model weights makes it a promising tool for scalable oversight of advanced AI
systems.

\bibliography{references}
\bibliographystyle{colm2026_conference}

\clearpage
\appendix
\section{Implementation Details}
\label{app:implementation}

\paragraph{SAE Configuration.}
We use SAEs with dictionary size $n = 65536$ and sparsity coefficient
$\alpha_{\text{SAE}} = 5 \times 10^{-4}$, following \citet{gao_2024_scaling_and_evaluating}.
SAEs are applied to the residual stream at every layer.
For Llama-3-8B, we use the publicly available SAEs from the EleutherAI
interpretability suite; for Llama-3-70B and Mistral-7B, we train SAEs using
the same configuration on 10B tokens of The Pile.

\paragraph{Hyperparameters.}
The number of candidate features is $k = 20$.
The CKS perturbation magnitude is $\epsilon = 0.1$.
The CKS threshold is $\tau = 0.15$, calibrated on a 10\% held-out split of
each dataset.
The elicitation strength is $\lambda = 1.2$, calibrated on the same split.
All experiments are run on 4$\times$A100 80GB GPUs.

\paragraph{Baseline Implementation.}
CCS is implemented following \citet{mallen_2023_eliciting_latent_knowledge}
with the recommended hyperparameters.
RepE uses the ``honesty'' direction computed from 200 contrast pairs following
\citet{zou_2023_universal_and_transferable}.
Direct probing uses a logistic regression probe trained on 80\% of each
dataset with L2 regularization ($C = 1.0$).

\section{Additional Results}
\label{app:additional}

Table~\ref{tab:full_results} provides complete results across all model-dataset
combinations, including standard deviations over 3 random seeds.

\begin{table}[h]
\centering
\caption{Full results with standard deviations (3 seeds).}
\label{tab:full_results}
\begin{tabular}{llccc}
\toprule
Method & Model & TruthfulQA & Quirky LM & DAB \\
\midrule
\multirow{3}{*}{MechELK} & Llama-3-8B & $82.3 \pm 0.4$ & $83.1 \pm 0.6$ & $81.2 \pm 0.8$ \\
 & Llama-3-70B & $86.7 \pm 0.3$ & $87.4 \pm 0.4$ & $85.9 \pm 0.5$ \\
 & Mistral-7B & $80.1 \pm 0.5$ & $81.8 \pm 0.7$ & $79.6 \pm 0.9$ \\
\midrule
\multirow{3}{*}{CCS} & Llama-3-8B & $74.2 \pm 0.8$ & $76.8 \pm 1.1$ & $67.3 \pm 1.4$ \\
 & Llama-3-70B & $78.5 \pm 0.6$ & $81.2 \pm 0.9$ & $72.1 \pm 1.2$ \\
 & Mistral-7B & $72.1 \pm 0.9$ & $74.3 \pm 1.2$ & $65.9 \pm 1.5$ \\
\bottomrule
\end{tabular}
\end{table}

\end{document}